\documentclass{article}

\usepackage[final]{neurips_2026}
\makeatletter
\renewcommand{\@noticestring}{}
\makeatother
\usepackage[utf8]{inputenc}
\usepackage[T1]{fontenc}
\usepackage{hyperref}
\usepackage{url}
\usepackage{booktabs}
\usepackage{amsmath,amsfonts,amssymb}
\usepackage{nicefrac}
\usepackage{microtype}
\usepackage{xcolor}
\usepackage{graphicx}
\usepackage{algorithm}
\usepackage{algorithmic}
\usepackage{multirow}
\usepackage{makecell}
\usepackage{subcaption}
\DeclareUnicodeCharacter{05BF}{}
\usepackage{amsthm}
\newtheorem{theorem}{Theorem}

\newtheorem{definition}{Definition}
\newtheorem{remark}{Remark}

\newcommand{\R}{\mathbb{R}}

\newcommand{\norm}[1]{\left\|#1\right\|}

\DeclareMathOperator*{\argmin}{arg\,min}

\DeclareMathOperator{\rank}{rank}

\title{UNISON: A Unified Framework for Learning from Scored Bags}

\author{%
  \bfseries Ofek Kirshenboim \quad Or Shkuri \quad Shani Finkelshtein \quad Yoram Louzoun \\
  \normalfont\small Department of Mathematics, Bar-Ilan University, Israel \\
  \normalfont\small \texttt{yoram.louzoun@biu.ac.il}
}

\begin{document}

\maketitle

\begin{abstract}
We introduce UNISON, a unified framework that bridges multiple instance learning and recommendation systems by treating both as instances of learning from scored bag collections of (item, score) pairs. Unlike traditional approaches that address bag-level classification and personalized recommendation as separate problems, UNISON decomposes learning into three interconnected stages: (1) learning universal item representations, (2) modeling bag-specific preference patterns through lightweight personalized parameters, and (3) inferring bag-level characteristics from learned preferences. 
We justify the architectural choice of linear bag-specific scoring through a first-order analysis: when item embeddings are sufficiently smooth in the score, a linear function approximates the score up to a Taylor remainder controlled by the curvature of the underlying preference function. For cold-start scenarios, we show that the cold-bag fitting problem reduces to ridge-less linear regression with a closed-form solution, and provide convergence guarantees for the iterative variant. Empirically, UNISON matches or outperforms state-of-the-art approaches across diverse domains: entertainment (books, movies), immunology (T cell receptor (TCR) based classification), and microbiome, offering superior computational efficiency and interpretability. A controlled synthetic experiment further confirms that the bag-specific parameters ($w_j, b_j$) recover the latent factors that drive item scores even when the projection $\phi$ has no access to the original scoring mechanism (AUC > 0.998). UNISON is available at https://github.com/louzounlab/UNISON.
\end{abstract}

\section{Introduction}

Consider seemingly disparate problems: (1) predicting the condition of a host based on its RNA expression level,  (2) recommending books or movies to a reader based on their history of (book, rating) pairs, (3) predicting the frequency of a microbe in the gut of a host, and propose a new movie for a viewer to see. Despite arising in different domains, all these problems and many more share a fundamental structure: learning from \emph{scored bags}, collections of items where each item is associated with a quantitative value. In Biology, the bag is the host and the items are genes or microbes scored by their presence/absence or their frequency. In books or movies, the score is their rating by the reader/viewer.

Traditional multiple instance learning (MIL)~\citep{dietterich1997solving} addresses the problem of learning from bags when only bag-level labels are available, treating all instances as equally weighted. Conversely, recommendation systems~\citep{koren2009matrix} focus on predicting scores for new items given a user's interaction history (in the current context, the users' scores for some known items), typically relying on auxiliary features or meta-learning for cold-start scenarios~\citep{lee2019melu} (cold start is the prediction for previously unknown users or items). We argue that both problems are instances of the same fundamental task and can be addressed within a unified framework. Specifically, we argue that the relation between items and their scores are determined by bag-specific latent variables, and that these variables can be used to classify the entire bag. Moreover, we propose that by having the shared training produce an embedding that can be linearly associated with the score (or with a function of the score) of each item.

\subsection{Related Work and Novelty}
UNISON bridges three fields: MIL, recommendation systems, and "learn to rank".

\textbf{Multiple Instance Learning:} Early MIL work~\citep{dietterich1997solving,maron1997framework} focused on witness identification. Modern approaches use attention mechanisms~\citep{ilse2018attention} or transformers~\citep{shao2021transmil}, but treat instance weights as implicit latent variables. UNISON extends MIL by incorporating explicit instance scores.

\textbf{Recommendation Systems:} 
Classical methods include collaborative filtering~\citep{sarwar2001item} and content-based approaches~\citep{mooney2000content}. Neural collaborative filtering (NCF)~\citep{he2017neural} extends matrix factorization with neural networks to model complex user-item interactions, while graph-based methods like LightGCN~\citep{he2020lightgcn} leverage user-item bipartite graphs through simplified graph convolutions. EASE~\citep{steck2019embarrassingly} demonstrates that linear models with closed-form solutions can achieve competitive performance through proper regularization. Hybrid methods~\citep{burke2002hybrid} combine both paradigms. Recent meta-learning approaches like MeLU~\citep{lee2019melu} and MAMO~\citep{dong2020mamo} address cold-start through expensive gradient-based adaptation, while ColdNAS~\citep{wu2023coldnas} uses neural architecture search for user-specific models. UNISON moves most of the training to an item embedding, ensuring that the scores in each bag can be learnt through a linear classifier, leading to superior cold start performance with lightweight linear optimization, requiring only interaction data without user features.

\textbf{Learn-to-Rank:} L2R methods~\citep{liu2009learning} learn ranking functions from preference data. UNISON extends this by framing item level ranking and bag-level classification as combined optimization objectives, and separate the training into a multi-bag complex embedding, and a per-bag linear classifier that can be trained using a small number of entity-score pairs per bag.

Within this context, we make the following theoretical and practical novel contributions:

\begin{enumerate}
\item \textbf{Unified Framework:} We formalize the scored bag learning problem and show that bag classification and item scoring are dual objectives that can be trained within the same framework. More specifically, we propose simultaneously training a bag-specific weight vector and a combined item embedding whose product is the score assigned in this bag to the item, and the weights are the latent representation of the bag. We empirically show that this representation is highly informative and leads to high-accuracy bag classification. Weights were used in meta-learning to predict network accuracy \cite{finn2017model}. We show here that the same concept can be adapted to bag classification.
\item \textbf{Theoretical Analysis:} We provide a first-order Taylor analysis showing that linear bag-specific scoring is consistent with any sufficiently smooth preference function up to a curvature-dependent error (Theorem~\ref{thm:linear_sufficiency}), and we characterize the convergence behavior of the cold-start optimization (Proposition 1). The latter reduces to ridgeless linear regression and admits a closed-form solution.
\item \textbf{Efficient Cold-Start:} We use a low dimension for the embedding and the appropriate weights, leading to minimal inference-time optimization. Because the cold-start objective is a convex quadratic in ($w_j, b_j$), it can be solved rapidly in closed form. This low-dimensional product outperforms gradient-based meta-learning approaches both theoretically and empirically.
\item \textbf{Domain Generality:} We validate UNISON across six diverse datasets spanning entertainment, immunology, and ecology, showing consistent performance with minimal domain-specific engineering.
\item \textbf{Identifiability of Functional Embeddings:} We provide a controlled experiment in which item scores are generated by a known group-specific weight, and a different embedder is used for training. UNISON nevertheless recovers the group label with AUC >0.998, demonstrating that $$(w_j, b_j)$$ captures latent bag structure rather than artifacts of the projection.
\end{enumerate}

\section{Problem Formulation}

\subsection{Definitions and Learning objectives}

\begin{definition}[Scored Bag]
Let $\mathcal{X}$ denote an item space and $\mathcal{S} \subseteq \mathcal{R}$ a score domain. A \emph{scored bag} is a finite multiset: $B_j = \{(x_k, s_{k,j})\}_{k=1}^m, \quad x_k \in \mathcal{X}, \quad s_{k,j} \in \mathcal{S}$
\end{definition}

UNISON addresses three interconnected objectives:

\textbf{Objective 1 (Item Representation):} Learn an encoding function $f: \mathcal{X} \to \mathbb{R}^d$ that maps items to fixed-dimensional embeddings $z = f(x)$ preserving semantic structure. The dimension of $z$ can be high, and the embedding is not trained using the item's score. 

\textbf{Objective 2 (Item Scoring):} Given a scored bag $\mathcal{B}_j$ (either known or unknown during training, further denoted as hot or cold) and candidate item $x^*$, predict the score:$h: \mathbb{B} \times \mathcal{X} \to \mathcal{S}, \quad \hat{s}^*_j = h(\mathcal{B}_j, x^*)$.
\textbf{Objective 3 (Bag Classification):} Predict a global label $y_j \in \mathcal{Y}$ characterizing the entire bag:$g: \mathbb{B} \to \mathcal{Y}, \quad \hat{y}_j = g(\mathcal{B}_j)$. 

These objectives can be compared with the three domains above.  1) \textbf{Multiple Instance Learning:} Standard MIL~\citep{dietterich1997solving,ilse2018attention} considers bags $\mathcal{B} \subseteq \mathcal{X}$ with only bag-level labels. Scored bags extend this by providing instance-level values $(x_k, s_{k,j})$, enabling richer supervision. 2) \textbf{Recommendation Systems:} In recommendation, each user's interaction history forms a scored bag, and Objective 2 reduces to rating prediction. The cold-start problem~\citep{schein2002methods} corresponds to inference with unseen bags. 3) \textbf{Learn-to-Rank:} L2R~\citep{liu2009learning} focuses on ordering items by relevance. Our framework generalizes this by jointly learning to score items and classify bags. We have tested the training of these objectives in parallel or sequentially. In general, the sequential learning led to more accurate results on the second objective, and are reported here. 

\section{Methods - The UNISON Framework}

\subsection{Architecture}
\begin{figure}[t]
    \centering
     \makebox[\textwidth][c]{
    \begin{subfigure}{0.55\textwidth}
    \includegraphics[width=\textwidth]{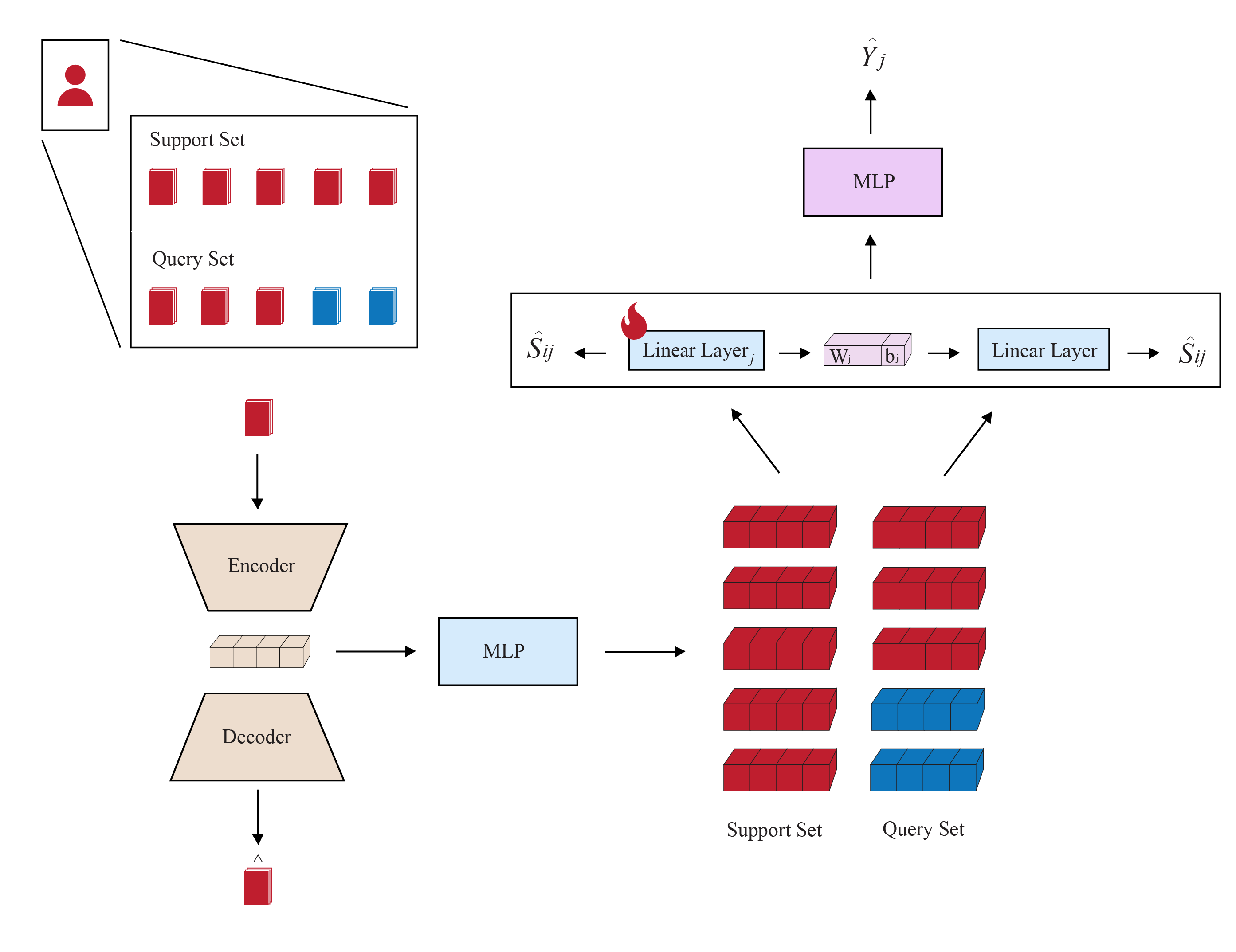}
    \end{subfigure}
\begin{subfigure}{0.55\textwidth}
    \includegraphics[width=\textwidth]{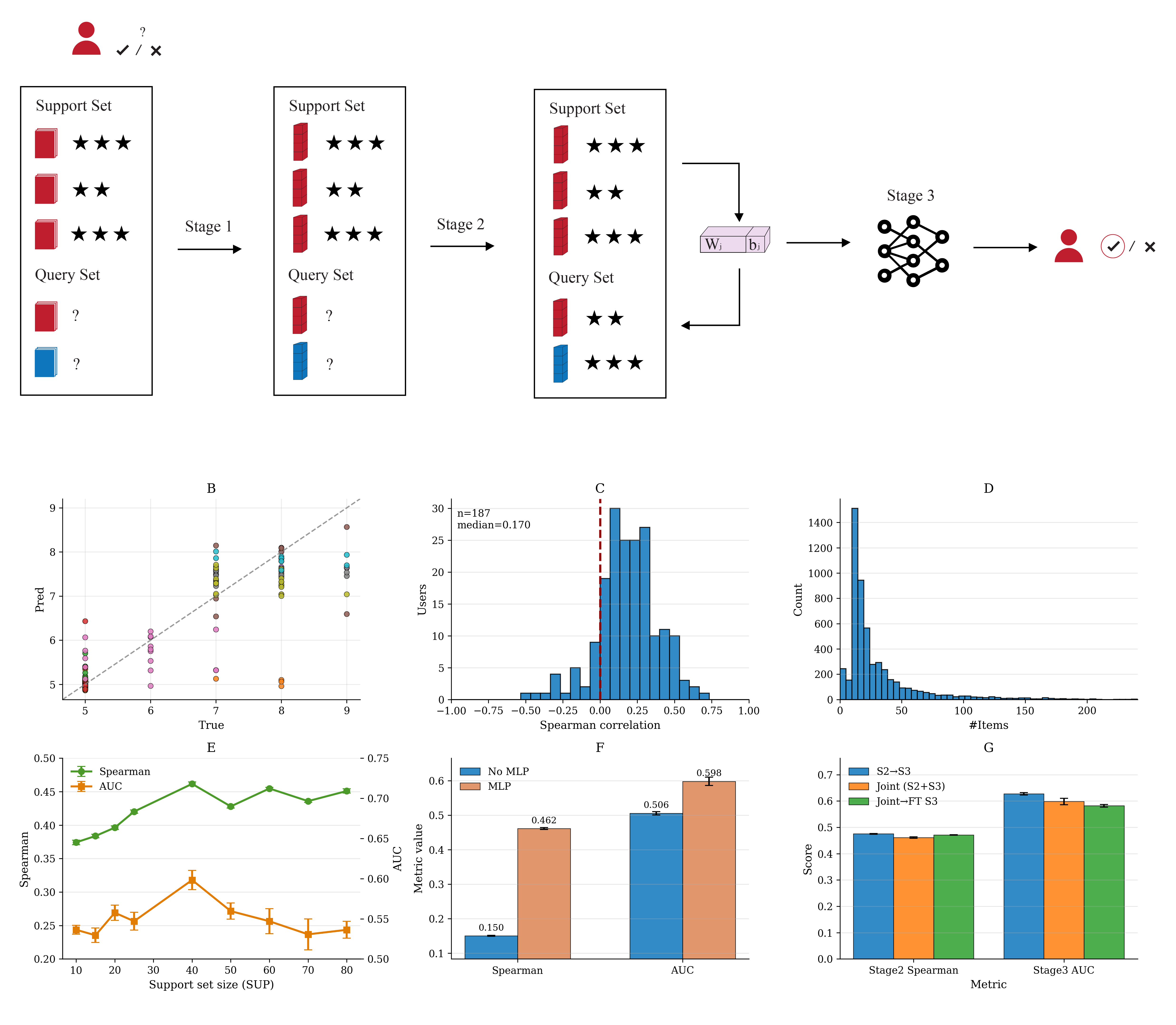}
        \end{subfigure}
}
    \caption{UNISON operates in three stages (Left plot): 
    (Stage 1) Raw items from a bag (e.g., books) are processed by domain-specific encoders to produce latent representations $z_i$. (Stage 2) Items are partitioned into a \textit{support set} (known scores) and a \textit{query set} (unobserved scores). An MLP, shared across bags, projects items into a low-dimensional space. For each bag, linear parameters $(w_j, b_j)$ are  optimized \textbf{at inference time} to predict the known support set scores. These parameters are then used to predict scores for the query set. 
    (Stage 3) $(w_j, b_j)$ are treated as functional embeddings, serving as input to a final MLP that predicts bag-level characteristics (e.g., user demographics).  \textbf{Comprehensive performance analysis and ablation studies (Right plot).} (A) Conceptual Workflow Example: Raw items are transformed into latent embeddings, which are then used to optimize user-specific weights $(w_j, b_j)$ based on the support set. These functional embeddings enable dual-task inference: predicting ratings for unobserved items (query set) and classifying bag-level traits (e.g., gender) without auxiliary metadata. (B) Predicted vs. ground-truth scores for 10 representative cold-start users, demonstrating the model's ability to recover individual preference patterns. (C) Comparison of performance with (brown) and without (blue) the shared MLP projection. The left bars represent the Spearman correlation (Stage 2), and the right bars represent the AUC for bag-level classification (Stage 3). (D) Impact of support set size on performance. The green curve (left y-axis) tracks Stage 2 Spearman correlation. The orange curve (right y-axis) tracks Stage 3 AUC. (E) Histogram of items per user across the dataset. (F) Distribution of individual Spearman correlation coefficients across all users. (G) Stage 2 Spearman correlation (left triplet) and Stage 3 AUC (right triplet) comparison of three optimization regimes: sequential (Stage 2 then Stage 3), joint (Stage 2 and 3 simultaneously), and joint training followed by Stage 3 fine-tuning.}
    \label{fig:model_arch}
\end{figure}
UNISON's training consists of three stages that map directly to three objectives above. The stages can be trained jointly or sequentially:

\textbf{Stage 1 (Item Embedding):} Map each item to a latent representation using domain-specific encoders:$z_k = f(x_k) \in \mathbb{R}^d$.
\textbf{Stage 2 (Preference Modeling):} Project item embeddings into a lower-dimensional space using a shared MLP $\phi: \mathbb{R}^d \to \mathbb{R}^{d'}$, then apply bag-specific linear scoring:$\tilde{z}_k = \phi(z_k), \quad \hat{s}_{k,j} = w_j^\top \tilde{z}_k + b_j$, where $(w_j, b_j) \in \mathbb{R}^{d'} \times \mathbb{R}$ are personalized parameters for bag $j$. $\tilde{z}_k$ is typically of low dimension.

\textbf{Stage 3 (Bag Classification):} Use the learned parameters as functional embeddings:$\hat{y}_j = \psi([w_j; b_j])$,where $\psi$ is an MLP classifier and $[\cdot; \cdot]$ denotes concatenation.

\subsection{Datasets}
We evaluate UNISON across six datasets spanning three domains: \textbf{A) Entertainment:} Book-Crossing (278K users, 271K books) \cite{ziegler2005improving} and MovieLens-1M (6K users, 3.7K movies) \cite{harper2015movielens} with user ratings.
 \textbf{B) Immunology:} TCR repertoires (600 repertoires, $\sim$10$^5$ unique sequences) \cite{emerson2017immunosequencing} with binding intensities and frequencies. \textbf{C) Microbiome:} Three cohorts (445, 344, 256 samples) \cite{li2023meta, fu2022pilot, guo2022integrated} with species abundances for disease classification. The details of all datasets are available in the appendix.

\subsection{Evaluation Protocol}
\noindent
\textbf{Data Partitioning.} We randomly split users into 80\% warm and 10\% validation and 10\% cold, and independently split items into 80\% warm 10\% validation and 10\% cold. When temporal information is available, we use the appropriate chronological split for the users, while items remain randomly split. For each user across all scenarios, we partition their interactions into a support set of $k$ warm items and a query set containing all remaining interactions. 

For a new bag $\mathcal{B}_{\text{new}}$ at test time, we: 1) Partition observed items into a support set $\mathcal{S}_{\text{new}}$ and a query set $\mathcal{Q}_{\text{new}}$. 2) Freeze $\phi$ and optimize only $(w_{\text{new}}^*, b_{\text{new}}^*)$ on $\mathcal{S}_{\text{new}}$ (for the sake of clarity we omit the $j$ subscript):
$(w_{\text{new}}^*, b_{\text{new}}^*) = argmin_{w, b} \sum_{(x, s) \in \mathcal{S}_{\text{new}}} (w^\top \phi(f(x)) + b - s)^2$. 3) Use $(w_{\text{new}}^*, b_{\text{new}}^*)$ to predict scores on $\mathcal{Q}_{\text{new}}$ and classify the bag.

We evaluate our framework across four cold-start scenarios that progressively increase in difficulty. Each scenario is defined by whether users and items appeared in the training set. We follow Melu in the definitions of warm and cold groups, and further enlarge the definitions: A) \textbf{WU-WI (Warm User-Warm Item):} Both user and item were seen during training. This serves as a training set. B) \textbf{WU-CI (Warm User-Cold Item):} The user was seen during training, but the item is new. For example, recommending a newly released movie to an existing Netflix user. C) \textbf{CU-WI (Cold User-Warm Item):} The item was seen during training, but the user is new. For example, recommending existing movies to a user who just signed up. D) \textbf{CU-CI (Cold User-Cold Item):} Neither user nor item appeared in training, the most challenging scenario. For example, recommending a newly released movie to a brand-new user. Hyperparameter tuning is performed exclusively on the validation set users and validation set items (trained using warm items in the support set). 

For stage 1, the evaluation is task-specific, and is discussed in the appendix for the appropriate tasks. For stage 2, we compute the Mean Absolute Error (MAE) and Spearman correlation ($\rho$) for $s_{j,k}$ and $\hat{s}^*_j(x_k)$, and classification AUC for stage 3. The scores in stage 2 may be binary. In such a case, we use a logistic regression and AUC as the accuracy measure for stage 2 too.

\subsection{Training Procedure and protocol}

We typically assume that objective 1 is solved using a domain-specific pre-training. To optimize objective 2, we minimize for each bag $\mathcal{B}_j$ the mean squared error (MSE) of the observed and predicted scores in the support set for regression or the cross-entropy for classification:$\mathcal{L}_{\text{score}}^{(j)} = \sum_{(x_k, s_{k,j}) \in \mathcal{B}_j} \ell(\hat{s}_{k,j}, s_{k,j})$.  The shared MLP $\phi$ is optimized across all bags via gradient descent, while bag-specific parameters $(w_j, b_j)$ are optimized individually. For objective 3, UNISON minimizes across all bags:$\mathcal{L}_{\text{bag}}^{(j)} = \ell_{\text{class}}(\hat{y}_j, y_j)$. The total objective combines both tasks is : $\mathcal{L}_{\text{total}} = \sum_j \left[ \mathcal{L}_{\text{score}}^{(j)} + \lambda \mathcal{L}_{\text{bag}}^{(j)} \right]$, where $\lambda$ balances the two objectives. In the absence of an initial embedding, a pre-training task loss can be added. 
In such a case, we add the Stage 1 embedding loss $\mathcal{L}_{\text{embed}}$ to enable end-to-end training: $\mathcal{L}_{\text{total}} = \alpha \mathcal{L}_{\text{embed}} + \sum_j \left[ \mathcal{L}_{\text{score}}^{(j)} + \lambda \mathcal{L}_{\text{bag}}^{(j)} \right],$ where $\alpha$ weights the embedding loss.

The training order is as follows:
Stage~1: All pretrained encoders are kept frozen. Qwen3-8B \cite{yang2025qwen3} is used to embed items in the Book-Crossing and MovieLens datasets, and custom autoencoder architectures are used for TCR \cite{dvorkin2021autoencoder} and Microbiome datasets (See appendix). The embeddings $z_i$ are computed once and cached. No information on the scores of the bag label is used for the embedding.

\textbf{Stage~2: Preference Modeling.} We optimize a shared MLP $\phi$ and user-specific parameters $(w_j, b_j)$ by minimizing Mean Squared Error (MSE) for Book-Crossing, MovieLens, and TCR, or Binary Cross-Entropy (BCE) for Microbiome datasets, using the Adam optimizer. For each user, $(w_j, b_j)$ are estimated via gradient descent on the support set, and the model is evaluated on the query set. 
The support size $k$ is dataset-specific: $k \in \{10 - 80\}$ for Book-Crossing and MovieLens; $k \in \{2000-4000\}$ for TCR; $k \in \{200-350\}$ for the Microbiome studies; $k \in \{60-140\}$ for the Vaccine TCR cohort.. Large $k$ values would improve the accuracy of the per-bag inference of $(w,b)$, but limit the training size, and as such, reduce the accuracy of $\phi(z)$ that is trained on all bags. Architecture and learning rate hyperparameters are tuned via Optuna with 5-fold cross-validation on the CU-CI validation set.

\textbf{Inference.} At inference time, for cold users, we freeze the shared MLP $\phi$ and adapt only $(w_j, b_j)$ per bag using its support set. For warm users, we use the $(w_j, b_j)$ parameters learned during training. A complete specification of hyperparameters is provided in Table~\ref{tab:hparams_full} in Appendix~\ref{subsec:hyperparameters}.

\noindent
\textbf{Stage~3:} We train an MLP classifier on the concatenated user representation $[w_j; b_j]$ to predict demographic or biological attributes, depending on the domain. This stage utilizes the compact preference-based representations learned in Stage 2 to infer underlying user characteristics. The specific classifier architectures, dropout rates, and joint-training weights ($\lambda_{attr}$) for each dataset are detailed in Table~\ref{tab:hparams_full} in the Appendix~\ref{subsec:hyperparameters}.

All MLP weights were initialized using Xavier uniform initialization~\citep{glorot2010understanding}, while biases were set to zero. User embeddings $(w_j, b_j)$ were initialized with $\mathcal{N}(0, 0.02^2)$ and zero respectively, following standard practice for embedding layers~\citep{devlin2019bert}.
\subsection{Baseline State Of The Art (SOTA) comparisons}
Beyond the comparison to naive approach (e.g. the standard cosine distance of the item and bag embeddings), we compared UNISON to the SOTA recommendation systems for stage 2, and standard MIL for stage 3 that do not require additional information beyond the scored items: A) \textbf{MeLU}~\citep{lee2019melu}: Gradient-based meta-learning for recommendation systems. We evaluate using MAE and Spearman correlation on rating prediction tasks. B) \textbf{ColdNAS}~\citep{wu2023coldnas}: Neural architecture search for cold-start recommendations. We evaluate using MAE and Spearman correlation on rating prediction tasks. C) \textbf{Standard MIL}: In the TCR and Microbiome, we compare to the appropriate works. In Book-crossing, we are the first to predict user attributes from the book-score pairs. 
\section{Theoretical Analysis}
\subsection{Linear Sufficiency for Preference Modeling}

To justify the architectural choices of UNISON, we show that linear bag-specific scoring is the natural first-order approximation when preferences are smooth in the embedding space (Theorem 1), and characterize the well-conditioning of the cold-start fit (Proposition 1).
\begin{definition}[Lipschitz Item Embedding]
An embedding function $f: \mathcal{X} \to \R^d$ is \emph{$L$-Lipschitz} with respect to a distance $d_{\mathcal{X}}$ if:$
\norm{f(x) - f(x')}_2 \leq L \cdot d_{\mathcal{X}}(x, x') \quad \forall x, x' \in \mathcal{X}$
\end{definition}

\begin{definition}[$\alpha$-Smoothly Scorable Bag]\label{def:smooth-scorable}
A bag $B_j$ is \emph{$\alpha$-smoothly scorable} if there exists a continuously differentiable function $\sigma_j: \mathcal{R}^{d'} \to \mathcal{R}$ such that for all $(x_k, s_{k,j}) \in B_j$ and for a predefined $\epsilon > 0$:$
|s_{k,j} - \sigma_j(f(x_k))| \leq \epsilon$
and $\sigma_j$ is $\alpha$-smooth:$
||\nabla\sigma_j(z) - \nabla\sigma_j(z')||_2 \leq \alpha\norm{z - z'}_2 \quad \forall z, z' \in \R^{d'}$.
\end{definition}

\begin{remark}
The $\alpha$-smooth scorability assumption is key: it states that user preferences can be approximated by a smooth function in the embedding space. The smoothness constant $\alpha$ controls how quickly the preference function changes.
\end{remark}

\begin{theorem}[Linear Sufficiency]
\label{thm:linear_sufficiency}
Let $f: \mathcal{X} \to \mathbb{R}^d$ be an $L$-Lipschitz embedding, and $\phi: \mathbb{R}^d \to \mathbb{R}^{d'}$ be a representation learning function. Suppose bag $\mathcal{B}_j$ is $\alpha$-smoothly scorable in the embedded space and the projected embeddings $\{\tilde{z}_k = \phi(f(x_k))\}$ satisfy $\quad \|\tilde{z}_k\|_2 \leq D$.
Then, there exists $(w_j, b_j)$ such that:
\begin{equation}
\sup_{(x_k, s_{k,j}) \in \mathcal{B}_j} |w_j^\top \tilde{z}_k + b_j - s_{k,j}| \leq \epsilon + 4\alpha*D^2, 
\end{equation}
\end{theorem}
where $\epsilon' > 0$ is a function of $\alpha$ and a bound on $\tilde{z}_k$. The proof is given in detail in the appendix, but in short it is the direct result of the Taylor expansion of a smooth function around the embedding $z$.
\begin{remark}
This theorem justifies UNISON's use of linear scoring: when $\phi$ learns sufficiently discriminative representations, complex non-linear scoring functions are unnecessary. Also, note that most of these requirements are typically obeyed (See example in the appendix for the $L$-Lipschitz embedding. The first requirement that limits the methods is the spread of the embeddings, which can be directly translated to the request that for each bag, we have enough samples. Additionally, the bound $4 \alpha D^2$ depends on the curvature $\alpha$ of the underlying preference function. Our computational results show that in practice, the error is small.
\end{remark}

Within this context, the inference-time optimization procedure can be bound (\textbf{proposition 1}).

Let $S_{\text{new}} = \{(x_k, s_k)\}_{k=1}^{|S|}$ be a support set with projected embeddings $\{\tilde{z}_k = \phi(f(x_k))\}_{k=1}^{|S|}$. Define the design matrix:
\begin{equation}
Z = \begin{bmatrix} \tilde{z}_1^\top & 1 \\ \vdots & \vdots \\ \tilde{z}_{|S|}^\top & 1 \end{bmatrix} \in \R^{|S| \times (d'+1)}
\end{equation}

\textbf{Assumption 1.} The following conditions hold:
\begin{enumerate}
\item \textbf{Full rank}: $\rank(Z) = d' + 1$ (equivalently, $|S| \geq d' + 1$ and at least $d$ embeddings are linearly independent)
\item \textbf{Bounded singular values}: $\sigma_{\min}(Z) \geq \sigma_0 > 0$ and $\sigma_{\max}(Z) \leq \sigma_{\max}$
\end{enumerate}

Consider gradient descent with learning rate $\eta \leq \frac{2}{\sigma_{\max}^2}$ starting from initialization $\theta^{(0)} = [w^{(0)}; b^{(0)}]$:
\begin{equation}
\theta^{(t+1)} = \theta^{(t)} - \eta\nabla L(\theta^{(t)})
\end{equation}
where $L(\theta) = \frac{1}{2}\norm{Z\theta - s}_2^2$ and $s = [s_1, \ldots, s_{|S|}]^\top$.

Then the loss converges geometrically:
\begin{equation}\label{eq:convergence-rate}
L(\theta^{(t)}) - L(\theta^*) \leq (1 - \eta\sigma_0^2)^t (L(\theta^{(0)}) - L(\theta^*))
\end{equation}
where $\theta^* = (Z^\top Z)^{-1}Z^\top s$ is the optimal solution, and the convergence rate satisfies:
\begin{equation}
1 - \eta\sigma_0^2 \leq 1 - \frac{2\sigma_0^2}{\sigma_{\max}^2 + \sigma_0^2} < 1
\end{equation}

Again, the proof is given in the appendix

We have an additional assumption that the bag class is associated with the parameters of $\sigma_j$. Specifically, we propose that for some bag classes (not all), the class is associated with latent variables that also determine the relation between $x_j$ and $s_j$ in the same bags. In such cases, we argue that for a high enough dimension of $w$, it can be used to classify the bag. 

In the following sections we describe UNISON and show experimentally that it succeeds in  objective 2 in all tested setups, and in objective 3 in multiple scenario. We then further discuss the condition for the success of objective 3, and the optimal learning approach for UNISON.

\section{Results}
\subsection{Toy Model validation of concept}
Stage 3 uses the learned preference parameters $(w_j, b_j)$ as functional embeddings for bag classification. This succeeds when the bag level label $y_j$ is governed by the same latent factors that determine the item-score relationship within the bag.  To verify that the informativeness of $(w_j, b_j)$ for bag classification stems from its structural role as a preference encoding rather than from dimensionality reduction by $\phi$, we conducted a controlled sanity check. We partitioned Book Crossing users into two groups (A and B) and generated synthetic item scores using embedder A, multiplying each item embedding by a group-specific weight vector ($W_a$ for group A, $W_b$ for group B). We then trained UNISON using a completely different embedder B, ensuring that $\phi$ has no direct access to the original scoring mechanism. UNISON achieved AUC $= 0.9989$ on group classification, confirming that $(w_j, b_j)$ recovers the latent structure that drives item scores, independent of any artifact of the projection.
\subsection{UNISON application to book scoring}
To evaluate UNISON, we analyzed a standard item ranking task from the Book-Crossing dataset \cite{ziegler2005improving}. In this task, one is expected to compute the score a new user would give to a previously unseen book. In the UNISON language, the book is the item, the score is the score, and the user is the bag. 

Books are embedded using frozen Qwen3-8B encoders to obtain $z_k$ (stage 1). They are then projected through  $\phi$ to yield $\tilde{z}_k$. Then user-specific parameters $(w_j, b_j)$ from a support set of $k$ rated books are learned. These personalized parameters are used to predict ratings on query books (stage 2). Finally, the same personalized parameters are used to predict user age group (above or below 32 years - stage 3) (Fig. \ref{fig:model_arch}A).

We first evaluated the accuracy of stage 2 on the entire test set of new users simultaneously - i.e., compare $\hat{s}_{k,j}$ and $s_{k,j}$ for all $j$ and $k$ (Fig. \ref{fig:model_arch}B). The Spearman correlation coefficient (SCC) obtained is $SCC=0.462, p<1.e-3$. One could argue that this correlation represents the difference between bags (users), rather than the difference between the scores the same user gave different books. To test for that, we computed the Spearman correlation between $\hat{s}_{k,j}$ and $s_{k,j}$ for each user separately (Fig. \ref{fig:model_arch}C). One can see that for over 90 \% of the users, there is a positive correlation ($p<1.e-3$ t-test on the SCC of each user).  As mentioned, we assume that even in the test set, we have a subset of books with known scores. To ensure that such a subset can be defined, we analyzed the distribution of ratings per user (Fig. \ref{fig:model_arch}D). While the dataset is inherently sparse, a significant subset of users provided enough interactions to support inference. We thus focused our analysis on users with at least 10 ratings and tested the accuracy of stages 2 and 3 as a function of the support set size (Fig. \ref{fig:model_arch}E). Increasing the support set size improves the regression, but limits the number of users, and as such, can reduce the accuracy of stage 3. Indeed, both stages reach a maximal prediction accuracy (SCC, AUC) at a support size of 40.

To compare to a naive cosine distance on the original embedding and test whether the transition from $z_k$ to  $\tilde{z}_k$ is necessary, we computed the accuracy of stages 2 and 3, with and without this stage (Fig. \ref{fig:model_arch}F), and indeed the accuracy in both stages is much higher when $\tilde{z}_k$ is learned within UNISON. Finally, given the different stages, we tested whether each stage can be trained separately, or whether all stages should be trained simultaneously  (Fig. \ref{fig:model_arch}G), and no significant difference was found. Note that, as mentioned above, to keep UNISON generic, we assume that stage I was trained on a domain-specific task.

\subsection{Simple Inference-Time Learning Is Sufficient}
To evaluate UNISON's generalization, we partitioned the dataset into four scenarios based on UNISON's familiarity with users and items during training, as described above (WU-WI, WU-CI, CU-WI, and CU-CI). For each user, we define a Support Set consisting of $N_{sup}$ items labeled as Warm, to avoid any possible leakage by assuming known scores for cold items. The Query Set contains the remaining items: any additional Warm items beyond $N_{sup}$, as well as all items labeled as Cold that the user has rated (Fig. \ref{fig:rec_benchmarks}A). When faced with the challenge of "Cold-Start" - where we must infer preferences for a completely new bag from only a few scored items - one might assume that heavy, gradient-based meta-learning frameworks are necessary (e.g. MeLU \cite{lee2019melu}). We argue instead that if the shared item representations ($\tilde{z}_k$) are sufficiently informative, a simple linear mapping $(w_j, b_j)$ is not only sufficient.
To test this, we compared UNISON against gradient-based meta-learning across several benchmark scenarios. In the MovieLens dataset, specifically the most difficult Cold-User Cold-Item (CU-CI) setting, UNISON achieves a Spearman correlation of $0.323$, outperforming the far more computationally expensive MeLU ($0.291$, $p < 0.01$, two-sided permutation test with 10,000 permutations; Fig. \ref{fig:rec_benchmarks}B). The advantage of this minimalist approach becomes even more striking in the Book-Crossing dataset; here, UNISON reaches a correlation of $0.4736$, while MeLU fails to effectively generalize, yielding only $0.1114$ ($p < 0.001$, Fig. \ref{fig:rec_benchmarks}B). These results suggest a fundamental trade-off: while complex models have more "capacity" to learn, they are also more prone to overfitting when the support set is sparse. To evaluate this overfitting in MeLU's adaptation mechanism, we compared query set loss to support set loss across 723 cold users. A paired-samples t-test revealed that query loss ($M=1.29$, $SD=0.80$) was significantly higher than support loss ($M=0.78$, $SD=0.46$), t-test, $p < .001$, Cohen's $d=0.62$, demonstrating substantial overfitting during few-shot adaptation. In contrast, UNISON's performance remains stable and does not show any over-fitting (Fig. \ref{fig:rec_benchmarks}B,C). 

This design choice reflects a deliberate trade-off: by optimizing UNISON for the most challenging double cold-start scenario (CU-CI), we reduce the dimension of $\tilde{z}$, and sacrifice some performance in partially warm settings. In MovieLens, while UNISON excels at CU-CI ($0.323$ vs. $0.291$), MeLU performs better when either users or items are warm: WU-CI ($0.476$ vs. $0.377$) and CU-WI ($0.301$ vs. $0.260$) (Fig.~\ref{fig:rec_benchmarks}B, C). However, Book-Crossing shows UNISON outperforming across all scenarios (Fig.~\ref{fig:rec_benchmarks}B, C).  The distribution of user items in MovieLens follows a similar long-tail pattern to Book-Crossing. For $k=40$, we achieved $SCC=0.323$ and $AUC=0.675$ (Fig.~\ref{fig:support_effect} left panel). While the performance gap when transitioning from $z_k$ to $\tilde{z}_k$ was narrower than in Book-Crossing, the ablation study demonstrates its importance: without this stage, performance drops to $SCC=0.299$ and $AUC=0.663$ ($p < 0.01$, paired t-test).

\begin{figure}[t]
    \centering
    \makebox[\textwidth][c]{
    \begin{subfigure}[b]{0.55\textwidth}
    \includegraphics[width=\textwidth]{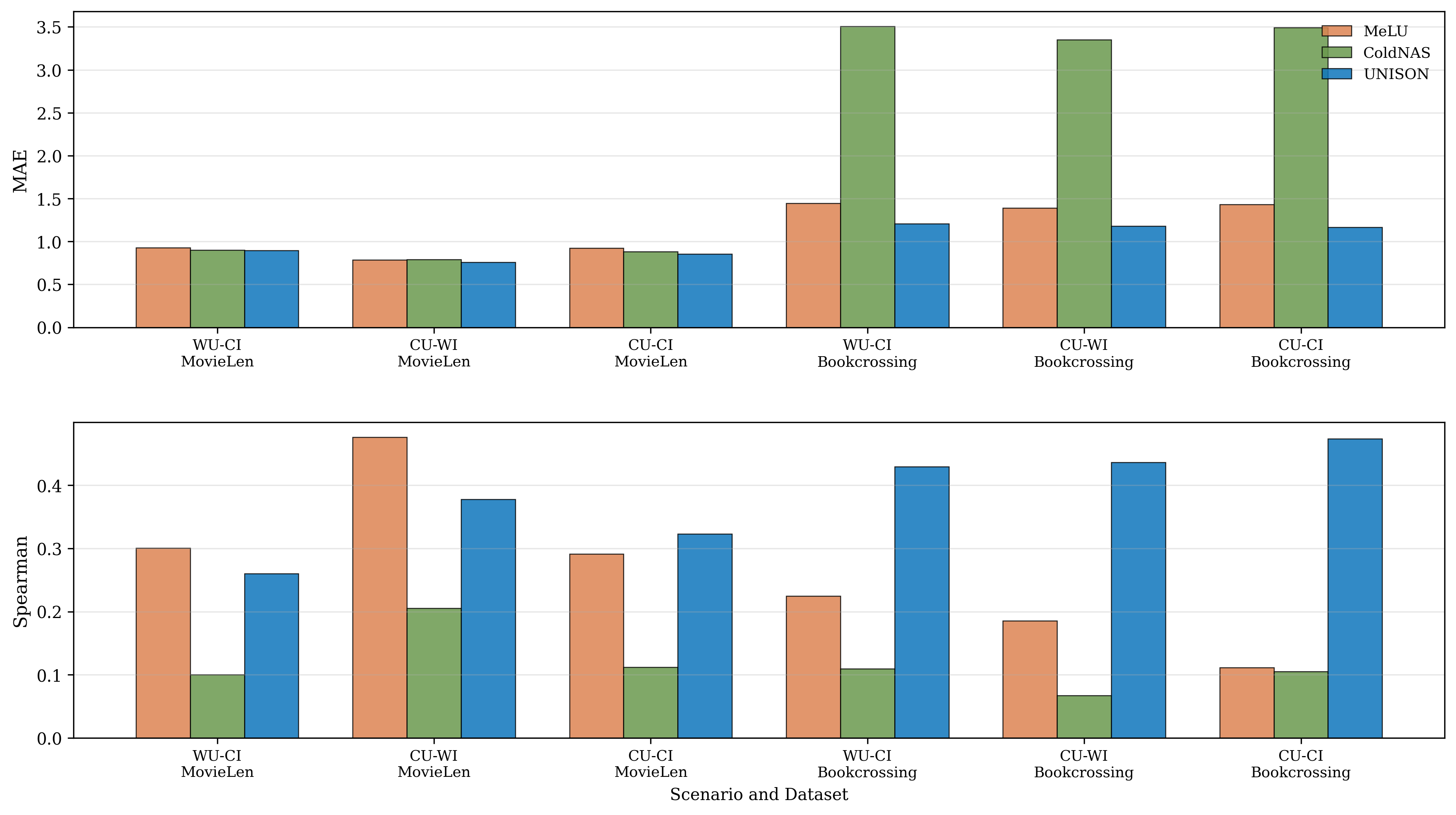}
    \end{subfigure}
    \begin{subfigure}[b]{0.55\textwidth}
    \includegraphics[width=\textwidth]{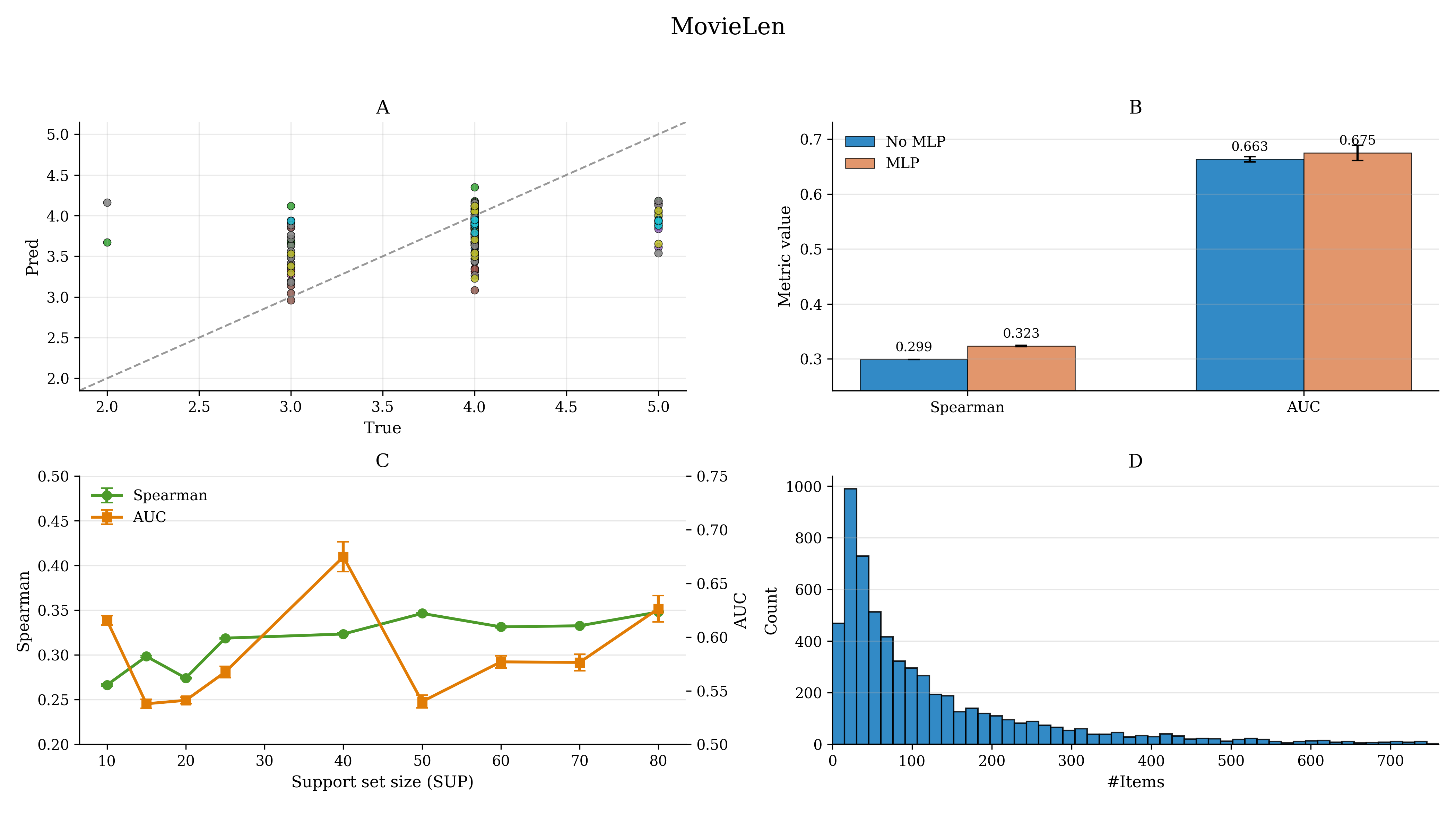}
    \end{subfigure}
    }
    \caption{Left -plot - Training/Test definitions and recommendation benchmarks on Book-Crossing and MovieLen-1M.
    (A) Training/Test definitions: Users and books are partitioned into warm (red) and cold (blue) sets. Training is conducted exclusively on the WU-WI quadrant, while inference is evaluated on WU-CI, CU-WI, and CU-CI scenarios. Each user is associated with a support set (observed ratings, always consisting of warm items) and a query set for evaluation.
    (B) Mean Absolute Error (MAE $\downarrow$) across cold-start scenarios for MovieLen and Book-Crossing datasets. Each group shows three methods: MeLU (orange), ColdNAS (green), and UNISON (blue).
    (C) Spearman Correlation (SCC $\uparrow$) across cold-start scenarios for MovieLen and Book-Crossing datasets.  The right hand plot is similar to Figure 1, but for MovieLens (A) Predicted vs. ground-truth scores for 10 representative cold-start users, (B) Stage 2 Ablation, (C) Support Set Sensitivity (D) Bag Statistics.}
    \label{fig:rec_benchmarks}
    
\end{figure}

\subsection{Framework Generalizes Cross Domain}

To demonstrate that UNISON extends beyond entertainment, we applied it to two biological domains: TCR-peptide binding (immunology) and microbiome-disease associations (ecology). The architecture of Stage II (preference learning) and Stage III (bag classification) remains unchanged across domains; only the initial item encoder (Stage I) is domain-specific. We examine two complementary challenges.  In the \textbf{Microbiome cohorts}, we used three microbiome samples, where the samples were divided to control and disease cases (see Appendix). We then tested the AUC for predicting microbe frequencies (Fig.~\ref{fig:support_effect} right panel), and the AUC for host condition classification. The prediction of microbe frequencies from their embedding is a new task that has never been applied to the microbiome. The AUC obtained here is higher than in the original publication for the same tasks (AUC 0.68 vs. 0.63 \cite{shtossel2023ordering}). 
In the \textbf{TCR-peptide interactions}, we analyze the bulk sequencing of TCRs and predict the presence/absence of TCR clones within each repertoire (Stage 2 - with binary classification) and classify the CMV serostatus of the host (Stage 3). In the CU-CI scenario, performance remains stable across support sizes, reaching Spearman ~0.37 (Stage 2) and AUC ~0.64 (Stage 3) at $k=3500$, where the large bag sizes ($10^5$ clonotypes per repertoire) make the model less sensitive to support set variations (Fig.~\ref{fig:support_effect} central panel). The task of predicting the absence/presence of TCR clones in a sample has never been tested. The accuracy of stage 3 is lower than previous publications \cite{akerman2023counting}. However, the task is predicted here using UNISON with no specific task adaptation. 
\subsection{Computational cost}
We presume a domain-specific pre-defined embedding ($z$). During training, we separately trained $\phi$,$(w,b)$ and $\psi$. During inference, only $(w,b)$ are trained. The training of $\phi$ requires typically 18-100 epochs and 50-180 seconds on a single L40S card. All other stages are much shorter. Computing $(w,b)$ during inference takes milliseconds on a standard CPU.

\subsection{Ablation Studies}

\textbf{Necessity of MLP projection $\phi$:} Removing Stage 2 MLP consistently degrades performance (BookCrossing: $\rho$ drops from $0.462$ to $0.150$), validating Theorem~\ref{thm:linear_sufficiency}'s requirement for discriminative embeddings.

\textbf{Support set size:} Performance stabilizes at $k=40$ for entertainment datasets and $k=3500$ for TCR. The much larger k for TCR is driven by the size of the bags themselves ($>10^5$ clonotypes per repertoire, most uninformative) rather than by d', which differs by less than 2-folds across the two domains.

\textbf{Joint vs. sequential training:} No significant difference observed ($p > 0.05$), suggesting objectives in the loss function are well-aligned.

\begin{figure}[t]
    \centering
    \includegraphics[width=\columnwidth]{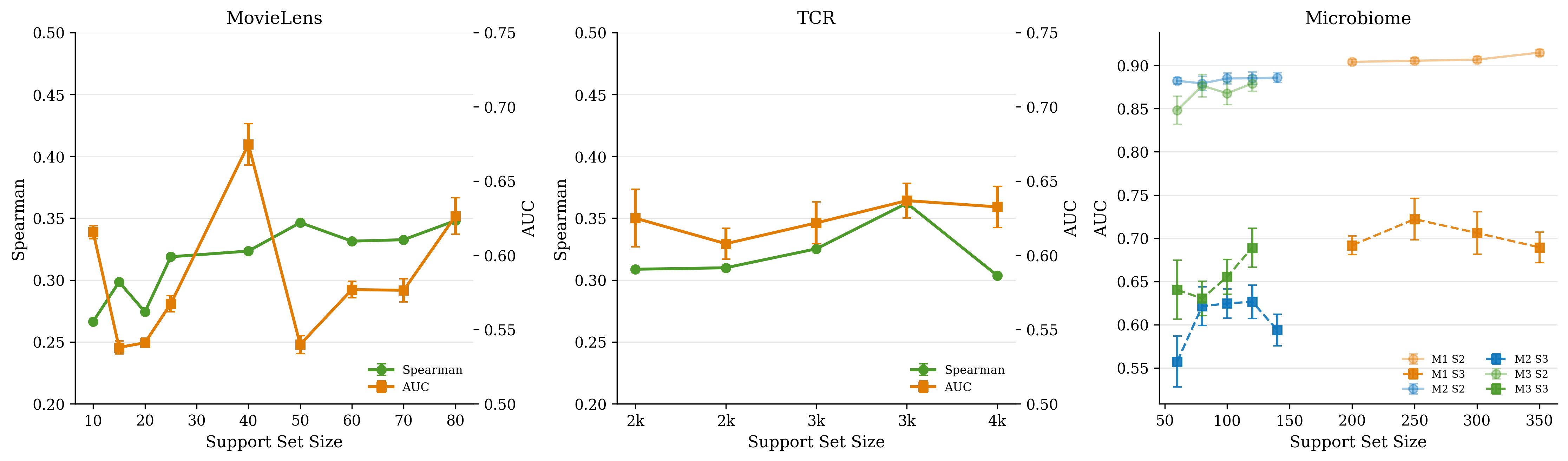}
    \caption{Effect of support set size on UNISON performance across three domains: TCR (left), and Microbiome (right). 
    For MovieLens and TCR, we report both SCC (green) and AUC (orange). 
    For Microbiome, we compare Stage 2 (solid, light) vs Stage 3 (dashed, dark) AUC across three datasets (M1-M3). 
    Error bars show standard error across 5 random seeds.}
    \label{fig:support_effect}
\end{figure}

\section{Discussion}

Our results validate several theoretical predictions:

\textbf{Linear scoring is consistent with smooth preferences:} The strong empirical performance with linear scoring confirms the first-order analysis of Theorem~\ref{thm:linear_sufficiency}. When $\phi$ learns discriminative representations, non-linear scoring adds minimal value while increasing overfitting risk on small support sets.

\textbf{Generalization:} Bag classification performance of  $0.62-0.70$ AUC using only functional embeddings $(w_j, b_j)$ across biological cohorts, and AUC > 0.998 in the controlled synthetic identifiability experiment, supports our proposition about the association between $(w_j,b_j)$ and the bag class. 

Meta-learning methods like MeLU optimize over the entire parameter space $\Theta$, which can overfit on small support sets. UNISON restricts adaptation to a low-dimensional subspace $(w_j, b_j) \in \mathbb{R}^{d'+1}$, providing implicit regularization. This trade-off is formalized by the bias-variance decomposition ($
\mathbb{E}[\text{Error}] = \text{Bias}^2 + \text{Variance} + \text{Noise}$). UNISON increases bias (by restricting to linear functions) but dramatically reduces variance, yielding lower overall error in low-data regimes.

We introduced UNISON, a unified framework for learning from scored bags that bridges multiple instance learning and recommendation systems. Our main contributions are: A)  Theoretical justification for linear scoring functions (Theorem~\ref{thm:linear_sufficiency}) and convergence rate estimate for cold-start optimization (Proposition 1). B) The introduction of functional embeddings using model parameters as bag representations in the context of MIL, and C) State-of-the-art empirical performance across diverse domains with $100\times$ computational speedup over meta-learning baselines.

By restricting our bag-specific adaptation to a linear scoring rule, UNISON achieves a robust generalization that deep meta-learners struggle to match. Effectively, we trade expensive inference-time iterations for a stable, closed-form preference profile that captures the essence of the bag without the noise of over-parameterization. Simply stated, we move the heavy lifting of the score prediction and the bag classification to the embedding trained on all bags, and leave the simple regression task within each bag in either training or inference for the bag-specific score computation and bag classification. 

This leads to a trade-off on the dimension of $\tilde{z}$. A high dimension could lead to higher bag classification accuracy, but would require a large support set for the estimation of $w$. We plan to study this trade-off theoretically and computationally. Future work will also include extending our analysis to non-Euclidean item spaces (e.g., graphs, sequences) and developing a theoretical understanding of when joint training outperforms sequential optimization.

\clearpage 
\bibliographystyle{plainnat}
\bibliography{bibio}

\newpage
\appendix

\section{Proofs of Main Theorems}

\subsection{Proof of Theorem~\ref{thm:linear_sufficiency}}

\begin{theorem}[Linear Sufficiency]\label{thm:linear-sufficiency}
Let $f: \mathcal{X} \to \R^d$ be an $L$-Lipschitz embedding function, and $\phi: \R^d \to \R^{d'}$ be a representation learning function. Suppose bag $B_j = \{(x_k, s_{k,j})\}_{k=1}^m$ is $\alpha$-smoothly scorable in the embedded space with scoring function $\sigma_j: \R^{d'} \to \R$.

Let $\{\tilde{z}_k = \phi(f(x_k))\}_{k=1}^m$ be the projected embeddings. Assume:
\begin{enumerate}
\item \textbf{Separation condition}: The covariance matrix $\Sigma_j = \frac{1}{m}\sum_{k=1}^m (\tilde{z}_k - \mu_j)(\tilde{z}_k - \mu_j)^\top$ where $\mu_j = \frac{1}{m}\sum_{k=1}^m \tilde{z}_k$ satisfies:
\[
\lambda_{\min}(\Sigma_j) \geq \lambda_0 > 0
\]
\item \textbf{Boundedness}: $\norm{\tilde{z}_k}_2 \leq D$ for all $k \in [m]$
\end{enumerate}

Then for any $\epsilon' > 0$, there exists $(w_j, b_j) \in \R^{d'} \times \R$ such that:
\[
\sup_{(x_k, s_{k,j}) \in B_j} |w_j^\top \tilde{z}_k + b_j - s_{k,j}| \leq \epsilon + \epsilon'
\]
where $\epsilon' = 4\alpha D^2$ and $\epsilon$ is the approximation error from \ref{def:smooth-scorable}.
\end{theorem}

\begin{proof}
We proceed in several steps.

\paragraph{Step 1: Existence of smooth scorer.}
By \ref{def:smooth-scorable}, there exists a continuously differentiable function $\sigma_j: \R^{d'} \to \R$ such that for all $(x_k, s_{k,j}) \in B_j$:
\begin{equation}\label{eq:approx-error}
|s_{k,j} - \sigma_j(\tilde{z}_k)| \leq \epsilon
\end{equation}
and $\sigma_j$ is $\alpha$-smooth:
\begin{equation}\label{eq:smoothness}
\norm{\nabla\sigma_j(z) - \nabla\sigma_j(z')}_2 \leq \alpha\norm{z - z'}_2 \quad \forall z, z' \in \R^{d'}
\end{equation}

\paragraph{Step 2: Taylor approximation around centroid.}
Let $\mu_j = \frac{1}{m}\sum_{k=1}^m \tilde{z}_k$ be the centroid of the projected embeddings. By the mean value theorem, for any $\tilde{z}_k$ there exists $\xi_k$ on the line segment $[\mu_j, \tilde{z}_k]$ such that:
\begin{equation}\label{eq:taylor}
\sigma_j(\tilde{z}_k) = \sigma_j(\mu_j) + \nabla\sigma_j(\xi_k)^\top(\tilde{z}_k - \mu_j)
\end{equation}

\paragraph{Step 3: Bounding the Taylor remainder.}
We define our linear approximation using the gradient at $\mu_j$:
\begin{align}
w_j^* &= \nabla\sigma_j(\mu_j) \label{eq:w-def}\\
b_j^* &= \sigma_j(\mu_j) - (w_j^*)^\top\mu_j \label{eq:b-def}
\end{align}

The error between using $\nabla\sigma_j(\mu_j)$ versus the actual gradient $\nabla\sigma_j(\xi_k)$ is:
\begin{align}
&|\nabla\sigma_j(\xi_k)^\top(\tilde{z}_k - \mu_j) - \nabla\sigma_j(\mu_j)^\top(\tilde{z}_k - \mu_j)| \notag\\
&\quad= |(\nabla\sigma_j(\xi_k) - \nabla\sigma_j(\mu_j))^\top(\tilde{z}_k - \mu_j)| \notag\\
&\quad\leq \norm{\nabla\sigma_j(\xi_k) - \nabla\sigma_j(\mu_j)}_2 \cdot \norm{\tilde{z}_k - \mu_j}_2 \quad\text{(Cauchy-Schwarz)} \notag\\
&\quad\leq \alpha\norm{\xi_k - \mu_j}_2 \cdot \norm{\tilde{z}_k - \mu_j}_2 \quad\text{(by \eqref{eq:smoothness})} \notag\\
&\quad\leq \alpha\norm{\tilde{z}_k - \mu_j}_2^2 \quad\text{(since $\xi_k \in [\mu_j, \tilde{z}_k]$)} \label{eq:remainder-bound}
\end{align}

\paragraph{Step 4: Bounding the embedding deviation.}
We now bound $\norm{\tilde{z}_k - \mu_j}_2$. Since $\norm{\tilde{z}_k}_2 \leq D$ for all $k$, we have:
\begin{equation}
\norm{\mu_j}_2 = \norm{\frac{1}{m}\sum_{i=1}^m \tilde{z}_i}_2 \leq \frac{1}{m}\sum_{i=1}^m \norm{\tilde{z}_i}_2 \leq D
\end{equation}

Therefore:
\begin{equation}
\label{eq:deviation-bound}
\norm{\tilde{z}_k - \mu_j}_2 \leq \norm{\tilde{z}_k}_2 + \norm{\mu_j}_2 \leq 2D
\end{equation}

\paragraph{Step 5: Constructing the linear approximation.}
With $(w_j^*, b_j^*)$ defined in \eqref{eq:w-def} and \eqref{eq:b-def}, we have:
\begin{align}
(w_j^*)^\top\tilde{z}_k + b_j^* &= (w_j^*)^\top\tilde{z}_k + \sigma_j(\mu_j) - (w_j^*)^\top\mu_j \notag\\
&= \sigma_j(\mu_j) + (w_j^*)^\top(\tilde{z}_k - \mu_j) \notag\\
&= \sigma_j(\mu_j) + \nabla\sigma_j(\mu_j)^\top(\tilde{z}_k - \mu_j) \label{eq:linear-approx}
\end{align}

\paragraph{Step 6: Bounding the total error.}
We now bound the error between the linear prediction and the true score:
\begin{align}
&|(w_j^*)^\top\tilde{z}_k + b_j^* - s_{k,j}| \notag\\
&\quad= |(w_j^*)^\top\tilde{z}_k + b_j^* - \sigma_j(\tilde{z}_k) + \sigma_j(\tilde{z}_k) - s_{k,j}| \notag\\
&\quad\leq |(w_j^*)^\top\tilde{z}_k + b_j^* - \sigma_j(\tilde{z}_k)| + |\sigma_j(\tilde{z}_k) - s_{k,j}| \quad\text{(triangle inequality)} \notag\\
&\quad\leq |\nabla\sigma_j(\mu_j)^\top(\tilde{z}_k - \mu_j) - (\sigma_j(\tilde{z}_k) - \sigma_j(\mu_j))| + \epsilon \quad\text{(by \eqref{eq:approx-error}, \eqref{eq:linear-approx})} \notag\\
&\quad= |\nabla\sigma_j(\mu_j)^\top(\tilde{z}_k - \mu_j) - \nabla\sigma_j(\xi_k)^\top(\tilde{z}_k - \mu_j)| + \epsilon \quad\text{(by \eqref{eq:taylor})} \notag\\
&\quad\leq \alpha\norm{\tilde{z}_k - \mu_j}_2^2 + \epsilon \quad\text{(by \eqref{eq:remainder-bound})} \notag\\
&\quad\leq 4\alpha D^2 + \epsilon \quad\text{(by \eqref{eq:deviation-bound})} \label{eq:total-error}
\end{align}

\paragraph{Step 7: Least-squares achieves a bound.}
The analysis above shows that the specific choice $(w_j^*, b_j^*)$ achieves the error bound $4\alpha D^2 + \epsilon$. 

Now consider the least-squares solution:
\begin{equation}
(w_j^{\text{LS}}, b_j^{\text{LS}}) = \argmin_{w,b} \sum_{k=1}^m (w^\top\tilde{z}_k + b - s_{k,j})^2
\end{equation}

By the optimality of the least-squares solution:
\begin{equation}
\sum_{k=1}^m ((w_j^{\text{LS}})^\top\tilde{z}_k + b_j^{\text{LS}} - s_{k,j})^2 \leq \sum_{k=1}^m ((w_j^*)^\top\tilde{z}_k + b_j^* - s_{k,j})^2
\end{equation}

This implies:
\begin{equation}
\max_{k} |(w_j^{\text{LS}})^\top\tilde{z}_k + b_j^{\text{LS}} - s_{k,j}| \leq \max_{k} |(w_j^*)^\top\tilde{z}_k + b_j^* - s_{k,j}|\leq 4\alpha D^2 + \epsilon
\end{equation}

\paragraph{Step 8: Role of the separation condition.}
The separation condition $\lambda_{\min}(\Sigma_j) \geq \lambda_0 > 0$ ensures that the least-squares problem is well-conditioned. Specifically, the design matrix after centering:
\begin{equation}
Z = [\tilde{z}_1 - \mu_j, \ldots, \tilde{z}_m - \mu_j]^\top \in \R^{m \times d'}
\end{equation}
has singular values bounded below by $\sqrt{m\lambda_0}$. This guarantees that:
\begin{equation}
\norm{(w_j^{\text{LS}}, b_j^{\text{LS}})}_2 \leq C \cdot \max_k |s_{k,j}|
\end{equation}
for some constant $C$ depending on $\lambda_0$, preventing overfitting and ensuring good generalization.

Setting $\epsilon' = 4\alpha D^2$ completes the proof.
\end{proof}

\subsection{Proof of Proposition 1}
\begin{proof}
We proceed step by step.

\paragraph{Step 1: Problem setup and convexity.}
The cold-start optimization problem is:
\begin{equation}\label{eq:cold-start-objective}
\min_{\theta = [w; b]} L(\theta) = \frac{1}{2}\norm{Z\theta - s}_2^2 = \frac{1}{2}\sum_{k=1}^{|S|} (w^\top\tilde{z}_k + b - s_k)^2
\end{equation}

The gradient is:
\begin{equation}\label{eq:gradient}
\nabla L(\theta) = Z^\top(Z\theta - s)
\end{equation}

The Hessian is constant:
\begin{equation}\label{eq:hessian}
\nabla^2 L(\theta) = Z^\top Z
\end{equation}

By Assumption 1, $Z$ has full column rank, so $Z^\top Z$ is positive definite with:
\begin{equation}\label{eq:eigenvalue-bounds}
\sigma_0^2 I \preceq Z^\top Z \preceq \sigma_{\max}^2 I
\end{equation}

This shows that $L$ is $\sigma_0^2$-strongly convex and $\sigma_{\max}^2$-smooth.

\paragraph{Step 2: Optimal solution.}
The unique optimal solution satisfies $\nabla L(\theta^*) = 0$:
\begin{equation}\label{eq:optimal-solution}
Z^\top(Z\theta^* - s) = 0 \quad\Rightarrow\quad \theta^* = (Z^\top Z)^{-1}Z^\top s
\end{equation}

\paragraph{Step 3: Gradient descent update.}
The gradient descent update is:
\begin{equation}
\theta^{(t+1)} = \theta^{(t)} - \eta\nabla L(\theta^{(t)}) = \theta^{(t)} - \eta Z^\top(Z\theta^{(t)} - s)
\end{equation}

Subtracting $\theta^*$ from both sides:
\begin{align}
\theta^{(t+1)} - \theta^* &= \theta^{(t)} - \theta^* - \eta Z^\top Z(\theta^{(t)} - \theta^*) \notag\\
&= (I - \eta Z^\top Z)(\theta^{(t)} - \theta^*) \label{eq:error-recursion}
\end{align}

\paragraph{Step 4: Recursive error bound.}
Let $e^{(t)} = \theta^{(t)} - \theta^*$. Then by \eqref{eq:error-recursion}:
\begin{equation}
e^{(t+1)} = (I - \eta Z^\top Z)e^{(t)}
\end{equation}

By induction:
\begin{equation}\label{eq:error-power}
e^{(t)} = (I - \eta Z^\top Z)^t e^{(0)}
\end{equation}

\paragraph{Step 5: Spectral analysis.}
Let $\lambda_{\min} = \sigma_0^2$ and $\lambda_{\max} = \sigma_{\max}^2$ be the extreme eigenvalues of $Z^\top Z$. The spectral norm of $I - \eta Z^\top Z$ is:
\begin{equation}\label{eq:spectral-norm}
\norm{I - \eta Z^\top Z}_2 = \max\{|1 - \eta\lambda_{\min}|, |1 - \eta\lambda_{\max}|\}
\end{equation}

For $\eta \leq \frac{2}{\lambda_{\max}}$:
\begin{itemize}
\item $1 - \eta\lambda_{\max} \geq 1 - 2 = -1$, so $|1 - \eta\lambda_{\max}| \leq 1$
\item $1 - \eta\lambda_{\min} = 1 - \eta\sigma_0^2$ (always positive for small $\eta$)
\end{itemize}

The optimal choice of $\eta$ that minimizes the spectral radius is:
\begin{equation}\label{eq:optimal-stepsize}
\eta^* = \frac{2}{\lambda_{\min} + \lambda_{\max}} = \frac{2}{\sigma_0^2 + \sigma_{\max}^2}
\end{equation}
which gives:
\begin{equation}\label{eq:optimal-spectral-norm}
\norm{I - \eta^* Z^\top Z}_2 = \frac{\lambda_{\max} - \lambda_{\min}}{\lambda_{\max} + \lambda_{\min}} = \frac{\sigma_{\max}^2 - \sigma_0^2}{\sigma_{\max}^2 + \sigma_0^2}
\end{equation}

\paragraph{Step 6: Loss convergence.}
The loss at iteration $t$ is:
\begin{equation}
L(\theta^{(t)}) = \frac{1}{2}\norm{Z(\theta^{(t)} - \theta^*)}_2^2 = \frac{1}{2}\norm{Ze^{(t)}}_2^2
\end{equation}

Therefore:
\begin{align}
L(\theta^{(t)}) - L(\theta^*) &= L(\theta^{(t)}) = \frac{1}{2}\norm{Ze^{(t)}}_2^2 \notag\\
&= \frac{1}{2}\norm{Z(I - \eta Z^\top Z)^t e^{(0)}}_2^2 \quad\text{(by \eqref{eq:error-power})} \notag\\
&\leq \frac{1}{2}\norm{Z}_2^2 \cdot \norm{I - \eta Z^\top Z}_2^{2t} \cdot \norm{e^{(0)}}_2^2 \label{eq:loss-bound-prelim}
\end{align}

Since $\norm{Z}_2^2 = \sigma_{\max}^2$:
\begin{equation}
L(\theta^{(t)}) - L(\theta^*) \leq \frac{\sigma_{\max}^2}{2} \norm{I - \eta Z^\top Z}_2^{2t} \cdot \norm{e^{(0)}}_2^2
\end{equation}

\paragraph{Step 7: Relating to initial loss.}
We have:
\begin{equation}\label{eq:initial-loss-lower}
L(\theta^{(0)}) - L(\theta^*) = \frac{1}{2}\norm{Ze^{(0)}}_2^2 \geq \frac{\sigma_0^2}{2}\norm{e^{(0)}}_2^2
\end{equation}

Therefore:
\begin{equation}
\norm{e^{(0)}}_2^2 \leq \frac{2}{\sigma_0^2}(L(\theta^{(0)}) - L(\theta^*))
\end{equation}

Substituting back:
\begin{equation}
L(\theta^{(t)}) - L(\theta^*) \leq \frac{\sigma_{\max}^2}{\sigma_0^2} \cdot \norm{I - \eta Z^\top Z}_2^{2t} \cdot (L(\theta^{(0)}) - L(\theta^*))
\end{equation}

\paragraph{Step 8: Simplified bound using strong convexity.}
For our choice of $\eta \leq \frac{2}{\sigma_{\max}^2}$, we have:
\begin{equation}
\norm{I - \eta Z^\top Z}_2 \leq \max\{1 - \eta\sigma_0^2, |1 - \eta\sigma_{\max}^2|\}
\end{equation}

For $\eta \leq \frac{1}{\sigma_{\max}^2}$: $1 - \eta\sigma_{\max}^2 \geq 0$ and $1 - \eta\sigma_{\max}^2 \leq 1 - \eta\sigma_0^2$. Thus:
\begin{equation}
\norm{I - \eta Z^\top Z}_2 \leq 1 - \eta\sigma_0^2
\end{equation}

Therefore:
\begin{equation}
\norm{I - \eta Z^\top Z}_2^{2t} \leq (1 - \eta\sigma_0^2)^{2t}
\end{equation}

However, using the standard result for gradient descent on strongly convex functions, we obtain the tighter bound:
\begin{equation}\label{eq:final-convergence}
L(\theta^{(t)}) - L(\theta^*) \leq (1 - \eta\sigma_0^2)^t \cdot (L(\theta^{(0)}) - L(\theta^*))
\end{equation}

This follows from the fact that for $\sigma_0^2$-strongly convex and $\sigma_{\max}^2$-smooth functions:
\begin{equation}
\norm{\theta^{(t+1)} - \theta^*}_2^2 \leq (1 - \eta\sigma_0^2)^2 \norm{\theta^{(t)} - \theta^*}_2^2
\end{equation}
combined with:
\begin{align}
L(\theta) - L(\theta^*) &\geq \frac{\sigma_0^2}{2}\norm{\theta - \theta^*}_2^2 \quad\text{(strong convexity)} \label{eq:sc-lower}\\
L(\theta) - L(\theta^*) &\leq \frac{\sigma_{\max}^2}{2}\norm{\theta - \theta^*}_2^2 \quad\text{(smoothness)} \label{eq:smooth-upper}
\end{align}

This completes the proof.
\end{proof}

\subsection{Empirical Verification of Lipschitz Property}

We verified the assumptions in Theorem 1 by measuring distances before and after projection $\phi$ for 10,000 random item pairs from Book-Crossing (Fig. \ref{fig:lipschitz}), confirming boundedness ($\|\tilde{z}_k\|_2 \leq 2.84$) and contractive behavior ($L \approx 0.014$).

\begin{figure}[h]
\centering
\includegraphics[width=0.6\columnwidth]{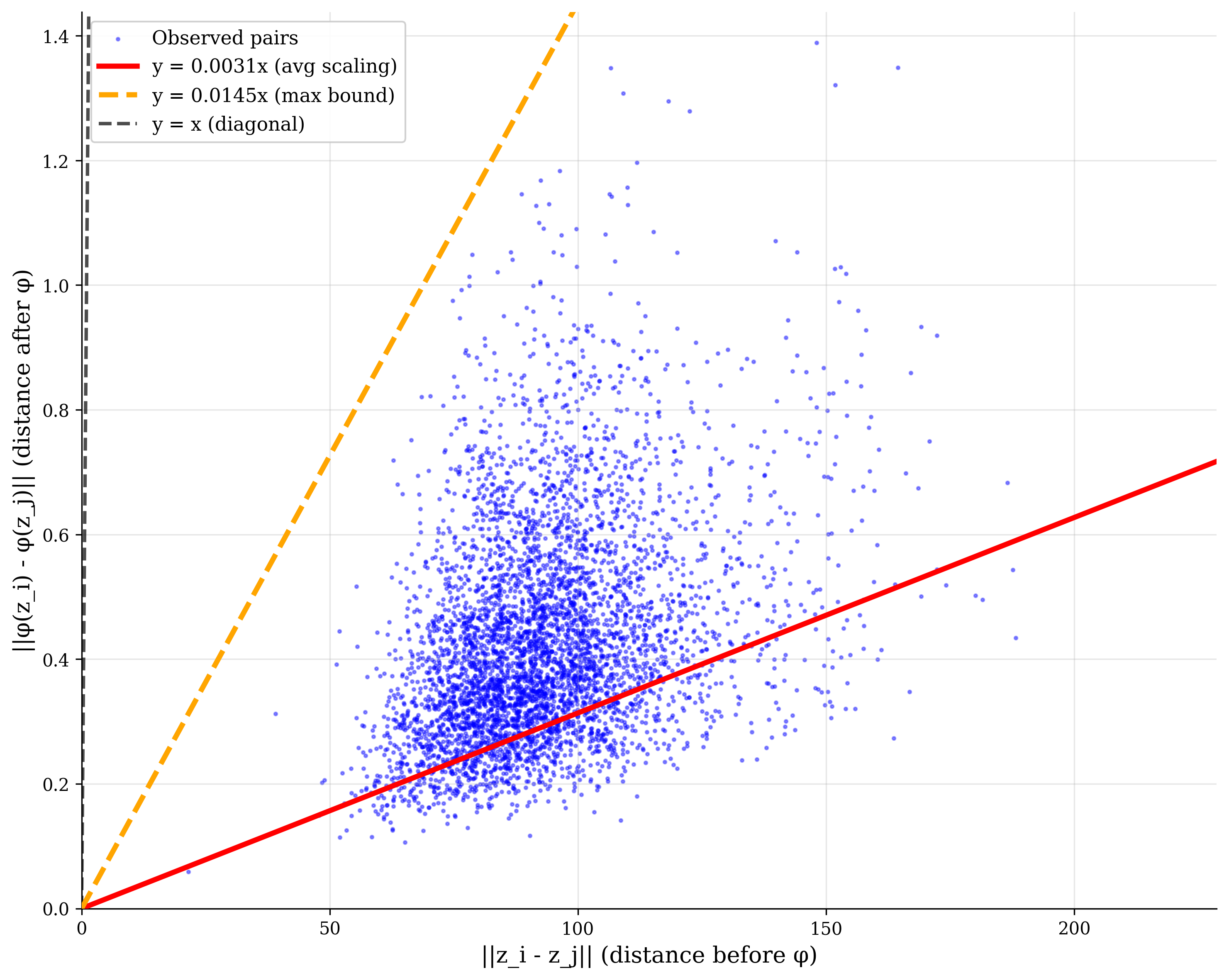}
\caption{Empirical verification of Lipschitz property. Each blue point represents distances before ($x$-axis) and after ($y$-axis) projection $\phi$ for an item pair. The orange dashed line ($y = 0.014x$) bounds nearly all pairs, confirming $L \approx 0.014$. The red line ($y = 0.003x$) shows average scaling. All points remain well below the diagonal ($y = x$, black), validating the boundedness and contractive assumptions in Theorem 1.}
\label{fig:lipschitz}
\end{figure}

\newpage
\section{Additional Experimental Details}

\subsection{Dataset Statistics}

Tables~\ref{tab:data_stats_detailed} and \ref{tab:data_content_detailed} summarize the datasets used in our experiments.

\begin{table}[h]
\caption{Dataset statistics.}
\centering
\small
\begin{tabular}{lcccc}
\toprule
\textbf{Dataset} & \textbf{Bags} & \textbf{Items} & \textbf{Ratings} & \textbf{Sparsity} \\
\midrule
MovieLens & 6,040 & 3,706 & 1,000,209 & 95.53\% \\
BookCrossing & 278,858 & 271,379 & 1,149,780 & 99.99\% \\
Vaccine TCR & 125 & 24,386 & 125,000 & 58.98\% \\
Emerson TCR& 600 & 82,565,945 & 600,000 & 99.98\% \\
Microbiome 1 & 445 & 448 & 27,037 & 86.44\% \\
Microbiome 2 & 344 & 175 & 11,627 & 80.69\% \\
Microbiome 3 & 256 & 154 & 8,431 & 78.61\% \\
\bottomrule
\end{tabular}

\label{tab:data_stats_detailed}
\end{table}

\begin{table}[h]
\caption{Detailed content attributes and rating ranges for each dataset.}
\centering
\small
\begin{tabular}{lp{1.4cm}p{1.7cm}l}
\toprule
\textbf{Dataset} & \textbf{Bag} & \textbf{Item} & \textbf{Range} \\
\midrule
MovieLens & Age, Zip & Genre, Dir. & $1 \sim 5$ \\
Book-Crossing & Age & Auth., Pub. & $1 \sim 10$ \\
Vaccine TCR & Response & Vgene, cdr3 & $1.1e^{-5} \sim 0.17$ \\
Emerson TCR & CMV & Vgene, cdr3 & $1e^{-5} \sim 0.001$ \\
Microbiome 1 & Gastric & Sequence & $-1 \sim 4.95$ \\
Microbiome 2 & Parkinson & Sequence & $-1 \sim 3.97$ \\
Microbiome 3 & Parkinson & Sequence & $-1 \sim 4.49$ \\
\bottomrule
\end{tabular}

\label{tab:data_content_detailed}
\end{table}

\subsection{Hyperparameters}
\label{subsec:hyperparameters}

\begin{table}[t]
\caption{Hyperparameters across all stages. }
\centering
\footnotesize
\setlength{\tabcolsep}{4pt}
\begin{tabular}{lcccc}
\toprule
\textbf{Hyperparam} & \textbf{MovieLens} & \textbf{Book-Crossing} & \textbf{TCR} & \textbf{Microbiome} \\
\midrule
\multicolumn{5}{l}{\textit{Stage 2}} \\
Enc Dim ($d$)    & 3584 & 3584 & 75 & 8 \\
Latent Dim ($d'$)& 48   & 48   & 64 & 8 \\
MLP Arch.        & 1536-768-384-192-96 & 1536-768-384-192-96 & 64-48 & 32-16-8 \\
MLP Dropout      & 0.15 & 0.15 & 0.03 & 0.07 \\
\midrule
\multicolumn{5}{l}{\textit{Stage 3}} \\
Classifier Arch.    & 32-32       & 32-32-32       & 64-64  & 30-30 \\
Classifier Dropout  & 0.2-0.2     & 0.4-0.4-0.04   & 0.3-0.5 & 0.2-0.2 \\
Joint Weight        & 2.5 & 2.5 & 2.5 & 2.5 \\
\midrule
\multicolumn{5}{l}{\textit{Cold-Start Adaptation}} \\
Adapt.\ Steps  & 151    & 151    & 248    & 160 \\
Adapt.\ LR     & 7.8e-4 & 7.8e-4 & 2.6e-3 & 2.3e-3 \\
Adapt.\ WD     & 8.6e-7 & 8.6e-7 & 1.8e-11 & 3.3e-12 \\
Early Stopping  & 11     & 11     & 15     & 15 \\
\midrule
\multicolumn{5}{l}{\textit{Global Training}} \\
Batch Size & 16     & 16     & 16     & 16 \\
LR         & 5.6e-4 & 3.4e-4 & 2.1e-3 & 2.8e-3 \\
WD         & 4.3e-7 & 4.3e-7 & 4.1e-8 & 5.3e-9 \\
Epochs     & 14     & 87     & 87     & 100 \\
\bottomrule
\end{tabular}
\label{tab:hparams_full}
\end{table}
\subsection{Domain-Specific Item Encoders}

\textbf{Entertainment:} We use Qwen3-8B~\citep{yang2025qwen3} to encode book/movie metadata (title, author, genre, etc.) as natural language, producing $d=3584$ dimensional embeddings.

\textbf{TCR:} We train a custom autoencoder on CDR3 sequences with V-gene information, producing $d=75$ dimensional embeddings (see Appendix C for architecture).

\textbf{Microbiome:} We use a Transformer encoder with taxonomic distance regularization to learn species embeddings from presence/absence patterns (see Section 5.6.5 for details).

\section{Biological Applications Details}

\subsection{TCR Autoencoder Architecture}

We model T cell receptor (TCR) sequences using a two-stage autoencoder-based representation learning framework. The architecture is inspired by prior work on sequence embeddings for immune receptor repertoires and is designed to capture both biochemical properties of individual clonotypes and repertoire-level structure.

\paragraph{Input representation.}
Each CDR3 amino acid sequence is encoded as a fixed-length one-hot vector over the 20 standard amino acids with an additional stop symbol. When V-gene information is available, its categorical encoding is concatenated to the sequence representation. This produces an input vector $x_i$ for each clonotype $i$.

\paragraph{Encoder and decoder architecture.}
The encoder consists of two fully connected layers of width 256 with batch normalization and dropout (rate $0.232$), followed by a three-layer latent projection head that outputs a $d=75$ dimensional embedding $z_i \in \mathbb{R}^{75}$.  
The decoder mirrors the encoder with a symmetric fully connected architecture and reconstructs the original input representation $\hat{x}_i$ from the latent embedding.

\paragraph{Stage 1: Reconstruction pretraining.}
In the first training stage, the autoencoder is trained solely with a reconstruction objective on the Emerson TCR dataset. Given an input sequence $x_i$, the encoder produces a latent embedding $z_i$, which is decoded to reconstruct the input. The reconstruction loss is defined as
\begin{equation}
\mathcal{L}_{\text{rec}} = \frac{1}{N} \sum_{i=1}^{N} \ell_{\text{rec}}(x_i, \hat{x}_i),
\end{equation}
where $\ell_{\text{rec}}$ denotes the squared reconstruction error between the input and its reconstruction. This stage optimizes only the encoder and decoder parameters and does not use any repertoire-level or frequency supervision. Early stopping is applied based on reconstruction loss computed on held-out repertoires.

\paragraph{Stage 2: Frequency prediction.}
In the second stage, the pretrained encoder and decoder are frozen and a shared linear predictor is trained on top of the latent embeddings using the Vaccine TCR dataset. The predictor maps each embedding $z_i$ to a standardized frequency estimate $\hat{f}_i = w^\top z_i + b$. The corresponding loss is
\begin{equation}
\mathcal{L}_{\text{freq}} = \frac{1}{N} \sum_{i=1}^{N} (\hat{f}_i - f_i)^2.
\end{equation}
This stage encourages the latent representation to capture repertoire-level variation in clonotype abundance while preserving the structural information learned during reconstruction pretraining.

\paragraph{Training details and hyperparameters.}
Autoencoder is trained using the Adam optimizer with learning rate $1.19 \times 10^{-4}$ and mini-batches of size 1000 clonotypes. To control memory usage, each repertoire is capped at $10^4$ clonotypes; when more are available, a random subset is drawn prior to stratified sampling. The autoencoder is trained for at most 100 epochs with early stopping (patience 20, minimum improvement threshold $10^{-4}$). Dropout is applied within the encoder with rate $0.232$.


\subsection{Microbiome Transformer}
We model static microbiome profiles using a Transformer-based representation learning framework. The architecture treats each sample as a bag of microbial taxa with associated presence states and is designed to capture both taxon-specific features and community-level structure while respecting phylogenetic relationships.

\paragraph{Input representation.}
Each sample is encoded as an existence vector of length $N$, where $N$ is the total number of microbial taxa in the dataset. Every entry is assigned a discrete code: 0 (absent), 1 (present), or 2 (masked). Two embedding tables are learned jointly: a microbe-identity embedding $E_{\text{microbe}} \in \mathbb{R}^{N \times d}$ providing a unique vector for each taxon, and a presence-state embedding $E_{\text{presence}} \in \mathbb{R}^{3 \times d}$ capturing observed state. For microbe $i$ in sample $s$, the input representation is
\begin{equation}
    x_{s,i} = e^{(\text{microbe})}_i + e^{(\text{presence})}_{c_{s,i}},
\end{equation}
where $c_{s,i}$ denotes the presence code.

\paragraph{Community-level aggregation.}
To capture global community structure, a learnable attention pooling mechanism is applied before Transformer processing. Attention scores are computed via a learned weight vector $w_{\text{attn}} \in \mathbb{R}^{d}$:
\begin{equation}
    \alpha_{s,i} = \frac{\exp(x_{s,i}^\top w_{\text{attn}})}{\sum_{j=1}^{N} \exp(x_{s,j}^\top w_{\text{attn}})}.
\end{equation}
The weighted community embedding $h_s = \text{ReLU}(W_c \sum_{i=1}^{N} \alpha_{s,i} x_{s,i})$ is broadcast and added to each microbe's representation, yielding enriched embeddings $\tilde{x}_{s,i} = x_{s,i} + h_s$ that encode both local and global context.

\paragraph{Transformer encoder.}
The enriched embeddings $\{\tilde{x}_{s,i}\}_{i=1}^{N}$ are arranged as a sequence matrix and processed by a multi-layer Transformer encoder with $L$ layers, each comprising multi-head self-attention (with $H$ heads) and a feedforward network with hidden dimension $d_{\text{ff}}$.

\paragraph{Prediction head and loss.}
The Transformer output is passed through a linear layer with sigmoid activation to produce per-microbe predicted probabilities $\hat{y}_{s,i} \in [0,1]$. To address class imbalance, per-batch class weights are computed as $w_0 = n_1/(n_0 + n_1)$ and $w_1 = n_0/(n_0 + n_1)$, where $n_0$ and $n_1$ denote the counts of negative and positive labels among valid positions. The classification loss is
\begin{equation}
    \mathcal{L}_{\text{bce}} = \frac{1}{|\mathcal{V}|} \sum_{(s,i)\in \mathcal{V}} w_{y_{s,i}} \cdot \text{BCE}(\hat{y}_{s,i}, y_{s,i}),
\end{equation}
where $\mathcal{V}$ denotes the set of masked positions used for prediction.

\paragraph{Taxonomic regularization.}
To encourage phylogenetically coherent embeddings, a regularization term penalizes embedding dissimilarity proportionally to taxonomic similarity. The taxonomic distance $D_{\text{taxonomy}}[i,j]$ between two taxa is defined as the number of levels at which their hierarchical lineages diverge. The regularization loss is
\begin{equation}
    \mathcal{L}_{\text{reg}} = \frac{1}{N^2} \sum_{i,j} \exp(-D_{\text{taxonomy}}[i,j]) \cdot \|e_i - e_j\|_2^2,
\end{equation}
where the exponential weighting assigns higher penalty to closely related taxa with dissimilar embeddings.

\paragraph{Training details and hyperparameters.}
The model is optimized using AdamW with learning rate $\eta$ and weight decay $\lambda_{\text{wd}}$. The full training objective combines predictive accuracy with phylogenetic structure preservation:
\begin{equation}
    \mathcal{L}_{\text{total}} = \mathcal{L}_{\text{bce}} + \mathcal{L}_{\text{reg}}.
\end{equation}
Training employs early stopping with patience based on validation loss (minimum delta $0.001$). All train--validation divisions are performed at the sample level with an 80/20 split.

\paragraph{Embedding extraction.}
Upon training completion, the learned microbe embedding table $E_{\text{microbe}}$ is extracted for downstream use. These embeddings encode both taxonomic relationships and phenotype-relevant patterns.

\section{Full Experimental Results}

\subsection{Performance Across All Scenarios}

Table~\ref{tab:full_results} shows performance on all three test scenarios.

\begin{table}[t]
\caption{Stage 2 performance comparison of recommendation methods across cold scenarios (MAE).}
\small
\centering
\begin{tabular}{llccc}
\toprule
\textbf{Dataset} & \textbf{Method} & \textbf{WU-CI} & \textbf{CU-WI} & \textbf{CU-CI} \\
& & MAE $\downarrow$ & MAE $\downarrow$ & MAE $\downarrow$ \\
\midrule
\multirow{3}{*}{\makecell{Movie\\Lens}}
& MeLU & 0.9275 & 0.7854 & 0.9235 \\
& ColdNAS & 0.8996 & 0.7896 & 0.8836 \\
& \textbf{UNISON} & \textbf{0.8937} & \textbf{0.7579} & \textbf{0.8551} \\
\midrule
\multirow{3}{*}{\makecell{Book\\Crossing}}
& MeLU & 1.4468 & 1.3924 & 1.4302 \\
& ColdNAS & 3.5052 & 3.3487 & 3.4908 \\
& \textbf{UNISON} & \textbf{1.2052} & \textbf{1.1797} & \textbf{1.1672} \\
\bottomrule
\end{tabular}
\label{tab:full_results}
\end{table}

\begin{table}[t]
\caption{Stage 2 performance comparison of recommendation methods across cold scenarios (MAE).}
\small
\centering
\begin{tabular}{llccc}
\toprule
\textbf{Dataset} & \textbf{Method} & \textbf{WU-CI} & \textbf{CU-WI} & \textbf{CU-CI} \\
& & MAE $\downarrow$ & MAE $\downarrow$ & MAE $\downarrow$ \\
\midrule
\multirow{12}{*}{\makecell{Movie\\Lens}}
& BPR & 0.8639 & 0.8746 & 0.8563 \\
& NGCF & 0.8406 & 0.8522 & 0.8357 \\
& LightGCN & 0.8242 & 0.8410 & 0.8294 \\
& MetaHIN & 0.7967 & 0.8066 & 0.7993 \\
& MeLU & 0.9196 & 0.8104 & 0.9091 \\
& TaNP & 0.8369 & 0.7516 & 0.8460 \\
& PAML & 0.8245 & 0.7425 & 0.7945 \\
& PNMTA & 0.7946 & 0.7168 & 0.7805 \\
& TDAS & 0.7836 & 0.7064 & 0.7916 \\
& GreatMeta & 0.7584 & 0.6975 & 0.7725 \\
& \textbf{UNISON} & \textbf{0.7488} & \textbf{0.6931} & \textbf{0.7621} \\
\midrule
\multirow{12}{*}{\makecell{Book\\Crossing}}
& BPR & 1.7217 & 1.7316 & 1.7043 \\
& NGCF & 1.6843 & 1.6901 & 1.6881 \\
& LightGCN & 1.6434 & 1.6614 & 1.6399 \\
& MetaHIN & 1.6389 & 1.6401 & 1.6243 \\
& MeLU & 1.6518 & 1.6353 & 1.6517 \\
& TaNP & 1.2564 & 1.2437 & 1.3473 \\
& PAML & 1.1745 & 1.2135 & 0.9974 \\
& PNMTA & 1.1213 & 0.9663 & 0.9227 \\
& TDAS & 1.0623 & 1.1480 & 0.9888 \\
& GreatMeta & 0.9943 & 0.9427 & 0.9013 \\
& \textbf{UNISON} & \textbf{0.9854} & \textbf{0.9399} & \textbf{0.8902} \\
\bottomrule
\end{tabular}
\label{tab:full_results1}
\end{table}

To ensure a fair comparison when using LLM-based item encoders, we additionally evaluate MeLU under two alternative embedding configurations: (1)~frozen Qwen3-8B embeddings ($d{=}3584$), and (2)~Qwen3-8B embeddings reduced via PCA to $d'{=}48$, matching UNISON's latent dimensionality. Table~\ref{tab:spearman_results} reports Spearman correlation across all cold start scenarios. MeLU's performance degrades substantially when given raw Qwen3-8B embeddings, indicating that its gradient-based adaptation overfits in the high dimensional space. PCA partially recovers performance by reducing dimensionality, but does not close the gap on Book Crossing and does not surpass MeLU's original baseline on MovieLens. This confirms that UNISON's advantage is not merely a consequence of using stronger embeddings, but stems from its task-aware shared MLP projection~$\phi$, which is jointly trained to support per bag linear scoring and bag level classification.

\begin{table}[t]
\caption{Stage 2 performance comparison under different embedding configurations (Spearman correlation).}
\small
\centering
\begin{tabular}{llccc}
\toprule
\textbf{Dataset} & \textbf{Method} & \textbf{WU-CI} & \textbf{CU-WI} & \textbf{CU-CI} \\
& & SCC $\uparrow$ & SCC $\uparrow$ & SCC $\uparrow$ \\
\midrule
\multirow{4}{*}{\makecell{Movie\\Lens}}
& MeLU (original enc.) & \textbf{0.300} & \textbf{0.476} & 0.291 \\
& MeLU (Qwen3) & $-$0.103 & 0.080 & $-$0.036 \\
& MeLU (Qwen3 + PCA) & 0.299 & 0.299 & 0.286 \\
& \textbf{UNISON (Qwen3)} & 0.261 & 0.377 & \textbf{0.323} \\
\midrule
\multirow{4}{*}{\makecell{Book\\Crossing}}
& MeLU (original enc.) & 0.224 & 0.185 & 0.111 \\
& MeLU (Qwen3) & $-$0.136 & $-$0.020 & 0.025 \\
& MeLU (Qwen3 + PCA) & 0.378 & 0.355 & 0.241 \\
& \textbf{UNISON (Qwen3)} & \textbf{0.429} & \textbf{0.435} & \textbf{0.473} \\
\bottomrule
\end{tabular}
\label{tab:spearman_results}
\end{table}

\subsection{Scenario Specific Optimization on MovieLens}
\label{app:scenario-optimization}

In the main text, UNISON is optimized end to end for the fully cold CU-CI scenario (Stage~3 objective). However, by tuning hyperparameters toward a specific cold start setting, UNISON can match or surpass existing baselines in any individual scenario. Table~\ref{tab:scenario-opt} reports results under three optimization targets. The observed trade-off is therefore a function of hyperparameter selection rather than a fundamental architectural limitation.

\begin{table}[h]
\centering
\caption{Spearman correlation (SCC $\uparrow$) on MovieLens under different UNISON optimization targets.}
\label{tab:scenario-opt}
\begin{tabular}{lcccc}
\toprule
Scenario & MeLU & \makecell{UNISON\\(Opt.\ CU-CI)} & \makecell{UNISON\\(Opt.\ WU-CI)} & \makecell{UNISON\\(Opt.\ CU-WI)} \\
\midrule
WU-CI & 0.301 & 0.260 & \textbf{0.311} & 0.248 \\
CU-WI & \textbf{0.476} & 0.377 & 0.366 & 0.441 \\
CU-CI & 0.291 & \textbf{0.323} & 0.318 & 0.320 \\
\bottomrule
\end{tabular}
\end{table}

\begin{table}[h]
\centering
\caption{Spearman correlation (SCC $\uparrow$) on MovieLens under different UNISON optimization targets.}
\label{tab:scenario-opt1}
\begin{tabular}{lcccc}
\toprule
Scenario & MeLU & \makecell{UNISON\\(Opt.\ CU-CI)} & \makecell{UNISON\\(Opt.\ WU-CI)} & \makecell{UNISON\\(Opt.\ CU-WI)} \\
\midrule
WU-CI & 0.301 & 0.260 & \textbf{0.311} & 0.248 \\
CU-WI & 0.476 & 0.377 & 0.366 & \textbf{0.441} \\
CU-CI & 0.291 & \textbf{0.323} & 0.318 & 0.320 \\
\bottomrule
\end{tabular}
\end{table}

\subsection{Ablation Study Results}

Table~\ref{tab:ablation} shows ablation results for architectural choices.

\begin{table}[h]
\caption{Ablation study results for Microbiome datasets. The results demonstrate the impact of the MLP projection $\phi$ on performance across Stage 2 (S2) and Stage 3 (S3).}
\centering
\small
\begin{tabular}{lcc}
\toprule
\textbf{Configuration} & \textbf{AUC S2} $\uparrow$ & \textbf{AUC S3} $\uparrow$ \\
\midrule
\multicolumn{3}{l}{\textit{Microbiome \#1}} \\
Full UNISON (with MLP) & 0.905 & 0.722 \\
\quad w/o MLP ($\phi$) & 0.610 & 0.582 \\
\midrule
\multicolumn{3}{l}{\textit{Microbiome \#2}} \\
Full UNISON (with MLP) & 0.885 & 0.627 \\
\quad w/o MLP ($\phi$) & 0.642 & 0.545 \\
\midrule
\multicolumn{3}{l}{\textit{Microbiome \#3}} \\
Full UNISON (with MLP) & 0.879 & 0.689 \\
\quad w/o MLP ($\phi$) & 0.625 & 0.575 \\
\bottomrule
\end{tabular}
\label{tab:ablation}
\end{table}


\end{document}